\documentclass{article}
\usepackage{kr}

\usepackage{times}
\usepackage{soul}
\usepackage{url}
\usepackage[hidelinks]{hyperref}
\usepackage[utf8]{inputenc}
\usepackage[small]{caption}
\usepackage{graphicx}
\usepackage{amsmath}
\usepackage{amsthm}
\usepackage{booktabs}
\usepackage{algorithm}
\usepackage{algorithmic}
\newtheorem{example}{Example}
\newtheorem{theorem}{Theorem}
\newtheorem{lemma}{Lemma}

\usepackage{amssymb}
\usepackage{tikz}
\usepackage[normalem]{ulem}

\theoremstyle{definition}
\newtheorem{proposition}{Proposition}

\newcommand{\SPO}{\mathcal{O}} 
\newcommand{\CHN}{\mathcal{C}} 
\newcommand{\TPRE}{\mathcal{T}} 

\newcommand{\p}{\pi} 
\renewcommand{\d}{\delta} 
\renewcommand{\o}{o} 
\newcommand{\cyc}[2]{\kappa_{#1}^{#2}} 
\newcommand{\cycfree}[2]{\alpha_{#1}^{#2}} 
\newcommand{\upper}[1]{\lfloor{#1}\rfloor} %
\newcommand{\upperd}[1]{\lfloor{#1}\rfloor^\mathsf{D}} %

\newcommand{\en}{\triangleright} 
\newcommand{\add}{{\operatorname{add}}} 

\newcommand{\dir}{\delta} 
\newcommand{\lvl}{\mathit{lvl}} 
\newcommand{\tightplus}[2]{{#1}{+}{#2}} 
\newcommand{\drc}[1]{({#1},\tightplus{#1}{1})} 
\newcommand{\as}{a} 

\newcommand{\ppp}[1]{\mathrm{P}_{#1}} 

\title{An AGM Approach to Revising Preferences}

\author{%
Adrian Haret$^1$\and
Johannes P. Wallner$^2$\\
\affiliations
$^1$ILLC, University of Amsterdam\\
$^2$TU Graz\\
\emails
a.haret@uva.nl,
wallner@ist.tugraz.at
}

\begin{document}

\maketitle

\begin{abstract}
    We look at preference change arising out of an interaction between two elements:
	the first is an initial preference ranking encoding a pre-existing attitude;
	the second element is new preference information signaling input from an authoritative source,
	which may come into conflict with the initial preference.
	The aim is to adjust the initial preference and bring it in line with the new preference, 
 	without having to give up more information than necessary.
 	We model this process using the formal machinery of belief change, 
 	along the lines of the well-known AGM approach.
	We propose a set of fundamental rationality postulates, 
    and derive the main results of the paper: a set of representation theorems 
 	showing that preference change according to these postulates 
    can be rationalized as a choice function guided by
 	a ranking on the comparisons in the initial preference order.
    We conclude by presenting operators satisfying our proposed postulates. 
 	Our approach thus allows us to situate preference revision within 
    the larger family of belief change operators.
\end{abstract}

\section{Introduction}\label{sec:introduction}
Preferences play a central role in theories of decision making,
as part of the mechanism underlying intentional behavior and rational choice,
both in economic models of rational agency 
as well as in formal models of artificial agents 
supposed to interact with the world and each other \cite{BoutilierBDHP04,DomshlakHKP11,RossiVW11,PigozziTV16}.
Since such interactions take place in dynamic environments, 
it can be expected that preferences change in response to new developments.

In this paper we are interested in preference change 
occurring when new preference information,
denoted by $\o$, becomes available and has to be taken at face value, 
thereby prompting a change in prior preference information, denoted by $\p$.
The change, we require, should preserve as much useful information from $\p$
as can be afforded.

Preference change thus described is a pervasive phenomenon, 
arising in many contexts spanning the realms of both human and artificial agency.
Thus, there is a distinguished tradition in Economics and Philosophy
that looks at examples of conflict between an agent's subjective preference
(what we call here the initial, or prior preference $\p$) 
and a second-order preference,
often standing for a commitment or moral rule 
(what we call here the new preference information $\o$):
subjective versus `ethical' preferences \cite{Harsanyi55},
lack of will, or \emph{akrasia} \cite{Jeffrey74},
moral commitments \cite{Sen77},
second-order volitions \cite{Frankfurt88} 
and second-order preferences \cite{Nozick94}
all fall under this heading.

The same challenge can occur in technological applications, 
from updating CP-nets~\cite{CadilhacALB15}
to changing the order in which search results 
are displayed on a page in response to user provided specifications.
At the same time, similar topics are emerging in the discussion 
on ethical decision making for artificial agents \cite{RossiM19}
and in issues related to the \emph{alignment problem} \cite{Russell2019}:
an artificial agent dealing with humans will have to learn their preferences, 
but as it cannot do so instantaneously, it must presumably acquire the relevant information 
in intermediate steps, revising along the way.

Thus, whether it is the internal conflict between 
an agent's private leanings and the better angels of its nature,
or a content provider wanting to tailor its products 
for a better user experience, 
many cases of preference change involve a conflict
between two types of preferences, one of which is 
perceived as having priority over the other.
However, even though the need to reconcile conflicting preferences in favor of one of them 
is widely acknowledged,
a concrete mechanism for resolving preference conflicts, that works for general preference 
orders, is often overlooked.

In keeping with prominent approaches to belief change, which 
model rational change using a plausibility relation over the states of affairs 
undergoing revision, and echoing a suggestion of Amartya Sen to the effect that conflicts among preferences 
can be understood using rankings over the preferences themselves \cite{Sen77}, 
we propose formalizing preference change using preferences over the basic elements 
of a preference order, as illustrated in the following example.

\begin{figure}\centering
	\begin{tikzpicture}
	\node at (0,-0.5){$\p$};
	\node at (0,1.8)(1){$1$};
	\node at (0,0.9)(2){$2$};
	\node at (0,0)(3){$3$};
	\path[-latex](1)edge(2)(2)edge(3)(1)edge[bend right,dotted](3);

	\node at (1,-0.5){$\o$};
	\node at (1,1.8)(1){$1$};
	\node at (1,0.9)(2){$2$};
	\node at (1,0)(3){$3$};
	\path[-latex](3)edge[bend right](1);

	\node at (3,-0.5){$\p_1$};
	\node at (3,1.8)(1){$1$};
	\node at (3,0.9)(2){$2$};
	\node at (3,0)(3){$3$};
	\path[-latex](3)edge[bend right](1)(1)edge(2)(3)edge[dotted](2);

	\node at (4,-0.5){$\p_2$};
	\node at (4,1.8)(1){$1$};
	\node at (4,0.9)(2){$2$};
	\node at (4,0)(3){$3$};
	\path[-latex](3)edge[bend right](1)(2)edge(3)(2)edge[dotted](1);

	\node at (5,-0.5){$\p_3$};
	\node at (5,1.8)(1){$1$};
	\node at (5,0.9)(2){$2$};
	\node at (5,0)(3){$3$};
	\path[-latex](3)edge[bend right](1);
	\end{tikzpicture}		
	\caption{
		Revising preference order $\p$ by $\o$: simply adding $\o$ to $\p$ leads to a cycle,
		so if $\o$ is accepted then a choice needs to be made regarding which of the initial comparisons of $\p$
		to keep;
		potential candidates for the revised order are $\p_1$, $\p_2$ or $\p_3$.
		A direct comparison ranking $i$ better than $j$ is	depicted by a solid arrow from $i$ to $j$, 
		with comparisons inferred by transitivity depicted by dotted arrows. }
	\label{fig:preferences-3-items}
\end{figure}
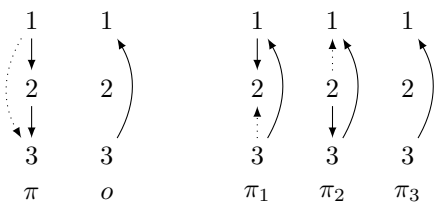

\begin{example}\label{ex:preferences-3-items}
	The initial preference $\p$ is such that,
	as a result of explicit assertion, item $1$ is ranked better than $2$ and $2$ is ranked better than $3$;
	by virtue of transitivity, it is also inferred that $1$ is considered better than $3$.
	We want to revise $\p$ by a preference $\o$, according to which $3$ is better than $1$
	(see Figure \ref{fig:preferences-3-items}). 
	The simplest solution is to add $\o$ to $\p$ 
	(i.e., include the comparisons contained in both),
	but the transitivity requirement leads to a cycle between $1$, $2$ and $3$, 
	which we would like to avoid.
	We are thus in a situation where $\p$ and $\o$ cannot be jointly accepted,
	but since $\o$, we stipulate, must be accepted,
	something must be given up from $\p$ (though,
	we ask, no more than strictly necessary). How is the decision to be made?
	We suggest that an implicit preference over the comparisons of $\p$ that were explicitly provided
	can provide an answer: if the comparison of $1$-vs-$2$ 
	(the edge from $1$ to $2$ in Figure \ref{fig:preferences-3-items})
	is preferred to the one of $2$-vs-$3$ then the result is $\p_1$, which
	holds on to $1$-vs-$2$ from $\p$ and together with $\o$ infers, by transitivity, 
	that $3$ is better than $2$;
	alternatively, a preference for $2$-vs-$3$ over $1$-vs-$2$
	leads to $\p_2$ as the result, while indifference between the two comparisons
	means that both are given up, resulting in $\p_3$.
	Thus, preference over comparisons in $\p$ translates as choice over how to go about revising $\p$.
	Interestingly, we may also reason in the opposite direction:
	observing choice behavior across different instances of revision allows us to infer preferences 
	over comparisons in $\p$, e.g., revising to $\p_1$, rather than to $\p_2$ or $\p_3$,
	can be rationalized as saying that the comparison of $1$-vs-$2$
	is considered better than $2$-vs-$3$.
\end{example}

\noindent
Our purpose here is to formalize the type of reasoning illustrated in Example \ref{ex:preferences-3-items}
by rationalizing preference change as a type of choice function on what we will call
the \emph{direct comparisons of $\p$},
i.e., the explicit preferences assumed to be given in $\p$.
Since a conflict between $\p$ and $\o$ forces some of the direct comparisons of $\p$
to be renounced, additional information in the form of a preference order over the direct comparisons
of $\p$ will serve as guide to the choice function.
The aim, in this, is not legislate on what is the right choice to make;
rather, it is to make sure that whatever the choice is, it is made in a coherent way.

\paragraph{Contributions.}
We present a mechanism for revising a preference order $\p$ 
that is based on an underlying preference relation over the basic, atomic 
comparisons of $\p$. This mechanism proceeds sequentially, 
by working its way through the underlying preference relation and adding 
as many of the direct comparisons of $\p$ as possible, while avoiding a conflict with $\o$.
We present a set of conditions under which the preference order on direct comparisons of $\p$ 
exists and has desired properties, 
and characterize the revision mechanism using a set of intuitive normative principles, 
i.e., rationality postulates in the AGM mould \cite{AlchourronGM85}.
The significance of our approach lies in laying bare the theoretical requirements
and basic assumptions for mechanisms intended to revise preferences.

\paragraph{Related work.}
Our work complements existing research, but manages to occupy a distinct niche in a broader landscape.
Some previous work labeled as preference revision \cite{Bradley07,LangT08,Liu11}, 
looks at changes in preferences prompted by a change in beliefs.
Here we abstract away from the source of the new information, choosing to focus 
exclusively on a mechanism that can be used for resolving conflicts:
the rational thing to do when knowing that, for some reason or other, one's preference has to change.
Other work \cite{CadilhacALB15} describes preference change when 
preferences are represented using CP-nets \cite{BoutilierBDHP04},
or dynamic epistemic logic \cite{BenthemL14}, 
in the context of declarative debugging \cite{DellP2005},
or databases \cite{Chomicki03},
and therefore comes 
with additional structural constraints.
In contrast, we have opted to represent preferences as 
strict partial orders over a set of items: we believe this straightforward 
formulation allows the basic issue signaled by Amartya Sen \cite{Sen77},
to be visible and tackled head on.

Apart from the issues raised in the Economics literature about second-order desires 
\cite{Harsanyi55,Jeffrey74,Sen77,Frankfurt88,Nozick94},
the basic phenomenon of preference change has also been raised
in explicit connection to belief change \cite{Hansson95,Grune-YanoffH09a,Grune-Yanoff13},
but a representation in terms of preferences on the comparisons present in 
the preference orders, along the lines suggested here, has, to the best of our knowledge, not yet been given.
Much existing work proceeds by putting forward some concrete preference revision mechanism, 
possibly by shifting some elements of the original preference around,
and occasionally with a remark on the similarity between this operation
and a belief revision operation \cite{Freund04,ChomickiS05,Liu11,MaBL12}. 
What our work adds to these models is an analysis in terms of postulates 
and representation results.

The postulates we put forward for preference revision bear a distinct resemblance to the AGM postulates 
employed for belief revision
\cite{AlchourronGM85,KatsunoM92,FermeH18}:
given that changing one's mind involves choosing some parts of a belief to keep and some to remove,
this is no coincidence. 
Indeed, the two problems are similar, though the structural particularities of preferences
(in particular, the requirement that they are transitive) 
mean that transfer of insights from belief revision to preference revision is by no means straightforward.

\paragraph{Outline.}
The rest of the paper is structured as follows. 
In Section \ref{sec:preliminaries}
we introduce notation and the basic elements of our model.
In Section \ref{sec:semantics} we provide a constructive way of revising a preference order,
based on rankings of the direct comparisons.
In Section \ref{sec:postulates} we provide a set of intuitive postulates together with their motivations, and discuss 
their appropriateness for the purpose of modelling preference revision.
In Section \ref{sec:coordination} we identify a set of conditions under which these postulates can be applied.
In Section \ref{sec:representation} we show that the postulates 
presented in Section \ref{sec:postulates} characterize the procedure described in Section \ref{sec:semantics}.
Section \ref{sec:concrete-operator} discusses concrete preference revision operators,
and Section \ref{sec:conclusion} offers concluding remarks.

\section{Preliminaries}\label{sec:preliminaries}
We assume a finite set $V$ of items, 
standing for the objects an agent can have preferences over.
If $\p$ is a binary relation on a set $V$ of items,
then $\p$ is a \emph{strict partial order (spo) on $V$}
if $\p$ is transitive and irreflexive,
and we write $\SPO_V$ for the set of strict partial orders on $V$.
If $\p$ is an spo on a set $V$ of items, 
then $\p$ is a \emph{strict linear order on $V$} if $\p$
is also total, in addition to being transitive and irreflexive.
A \emph{chain on $V$} is a strict linear order on a subset of $V$.
We write $\CHN_V$ for the set of chains on $V$.
FInally, $\p$ is a total preorder on $V$ if $\p$ is transitive and total, 
with $\TPRE_V$ being the set of total preorders on $V$. Note that in the following we will 
typically be interested in total preorders on $V\times V$, 
i.e., total preorders on the set of comparisons of items in $V$.

If $\p$ is an spo on a set of items $V$,
then a \emph{comparison $(i,j)$ of $\p$} 
is an element $(i,j)\in\p$, for some items $i,j\in V$,
interpreted as saying that, in the context of $\p$, 
$i$ is considered strictly better than $j$.
To simplify notation, 
we sometimes also refer to comparisons with the letter $c$.
We often have to consider the union $\p_1\cup\p_2$ of two spos,
which is not guaranteed to be an spo, 
since transitivity is not preserved under unions.
If this is the case, we typically have to 
substitute $\p_1\cup\p_2$ for its 
\emph{transitive closure}, denoted by $(\p_1\cup\p_2)^+$.
Since preferences are required to be transitive, we write a sequence of comparisons
$\{(1,2),(2,3)\dots,(m{-}1,m)\}^+$ as $(1,\dots, m)$.

If $\p = (i_1,\dots, i_m)$ is a chain on $V$,
a \emph{direct comparison of $\p$} is a comparison 
$(i_k,i_{k+1})\in\p$,
i.e., a comparison between $i_k$ and its direct successor in $\p$,
with $\dir_\p$ being the set of direct comparisons of $\p$.
The assumption is that direct comparisons are the result of explicit information, 
and are basic in the sense that they cannot 
be inferred by transitivity using other comparisons in $\p$.
Given preference orders $\p\in\CHN_V$ and $\o\in\SPO_V$, we want to carve out the
possible options for the revision of $\p$ by $\o$.
For this we use the set $\upper{\o}_\p$ of \emph{$\p$-completions of $\o$}, defined as:
$$
\upper{\o}_\p = \{(\o\cup\d)^+\in\SPO_V\mid\d\subseteq\dir_\p\}.
$$
The intuition is that a $\p$-completion of $\o$ is a preference order
constructed from $\o$ using some, and only, direct comparisons in $\p$, i.e., information
originating exclusively from the two sources given as input.
We will expect that a preference revision operator selects one element of this set as
the revision result.

Though taking $(\p\cup\o)^+$ as the result of revising $\p$ by $\o$ is not, in general, feasible,
we still want to identify parts of $(\p\cup\o)^+$ that are uncontroversial.
To that end, the \emph{cycle-free part $\cycfree{\p}{\o}$ of $(\p\cup\o)^+$} is defined as:
$$
\cycfree{\p}{\o} = \{\drc{i}\in(\p\cup\o)^+\mid(\tightplus{i}{1},i)\notin(\p\cup\o)^+\},
$$
i.e., the set of comparisons of $(\p\cup\o)^+$ not involved in a cycle with the comparisons of $\o$.
The \emph{cyclic part $\cyc{\p}{\o}$ of $\p$ with respect to $\o$} is defined as:
$$
\cyc{\p}{\o} = \{\drc{i}\in\dir_\p\mid(\tightplus{i}{1},i)\in(\p\cup\o)^+\},
$$
i.e., the set of direct comparisons of $\p$ involved in a cycle with $\o$.

\begin{example}\label{ex:completions}
	For $\p$ and $\o$ as in Example \ref{ex:preferences-3-items}, 
	we have that $\dir_\p=\{(1,2),(2,3)\}$,
	while the $\p$-completions of $\o$ are 
	$\upper{\o}_\p=\{(3,1,2), (2,3,1), (3,1)\}$,
	i.e., the spos obtained by adding to $\o$ either of the elements of $\dir_\p$, or none
	(corresponding to $\p_1$, $\p_2$ and $\p_3$).
	The cyclic part of $\p$ with respect to $\o$ is 
	$\cyc{\p}{\o}=\dir_\p=\{(1,2),(2,3)\}$
	and
	the cycle-free part of $\p$ with respect to $\o$ is
	$\cycfree{\p}{\o}=\emptyset$.
\end{example}

\section{A General Method for Revising Preferences}\label{sec:semantics}
A \emph{preference revision operator $\en$}
is a function
$\en\colon\CHN_V\times\SPO_V\rightarrow\SPO_V$ 
taking a chain $\p$ and an spo $\o$ as input, 
and returning an spo $\p\en\o$ as output.

The choice of input and output can be motivated by 
imagining that $\p$ stands for an existing priority ranking, 
e.g., the ordering of items on a webpage, 
whereas the new information $\o$ is provided by a user 
and is more likely to be incomplete.

In addition, we may look at this in light of the material that is to come: 
since we will be rationalizing preference revision operators using preferences (i.e., preorders) on comparisons,
an spo as output reflects the fact that certain comparisons are 
considered equally good, and must be given up together. 
The unfortunate effect of this is that the input and output formats do not match,
which makes it unclear, at this point, whether we can iterate the revision operation.
That being said, the output can (and will) be tightened to a chain: provided that the preferences 
guiding revision are a linear order (i.e., there are no ties). 
We touch on this aspect at the end of Section \ref{sec:representation}.

We start, then, by presenting a general procedure for revising preferences that, as advertised,
utilizes total preorders on the set $\dir_\p$ of direct comparisons of $\p$:
thus, a \emph{preference assignment $\as$} is a function 
$\as\colon\CHN_V\rightarrow\TPRE_{V\times V}$ mapping
every preference $\p\in\CHN_V$ to a 
total preorder $\le_\p$ on elements of $V\times V$, 
i.e., on pairwise comparisons on the items of $V$, 
of which we are interested only in the preorder on $\dir_\p$.
In typical AGM manner, a comparison $c_i\le_\p c_j$ 
in the context of a preorder $\le_\p$ on $\dir_\p$ means
that $c_i$ is \emph{better} than $c_j$.

If $\p\in\CHN_V$, $\o\in\SPO_V$ and $\le_\p$ is a total preorder on 	$\dir_\p$,
then, for $i\geq 1$, the \emph{$\le_\p$-level $i$ of $\dir_\p$}, denoted $\lvl_\le^i(\dir_\p)$,
contains the $i^\text{th}$ best elements of $\dir_\p$ according to $\le_\p$,
i.e., 
$\lvl^1_{\le_\p}(\dir_\p)=\min_{\le_\p}(\dir_\p)$,
$\lvl^{i+1}_{\le_\p}(\dir_\p)=\min_{\le_\p}(\dir_\p\setminus\bigcup_{1\leq j\leq i}\lvl^j_{\le_\p}(\dir_\p))$, etc.
Note that the $\le_\p$-levels of $\dir_\p$ partition $\dir_\p$ and, since $\dir_\p$ is finite,
there exists a $j>0$ 
such that $\lvl^i_{\le_\p}(\dir_\p)=\emptyset$, for all $i\ge j$.
The \emph{addition operator $\add^i_{\le_\p}(\o)$} is defined, 
for any $\o\in\SPO_V$ and $i\geq 0$, as follows:
\begin{align*}
\add_{\le_\p}^0(\o) &= (\o\cup\cycfree{\p}{\o})^+,\\
\add_{\le_\p}^i(\o) &= 
\begin{cases}
(\add^{i-1}_{\le_\p}(\o)\cup(\lvl_{\le_\p}^i(\dir_\p)\cap\cyc{\p}{\o}))^+,~\text{if in}~\SPO_V,\\
\add_{\le_\p}^{i-1}(\o),~\text{otherwise}.
\end{cases}
\end{align*}

\noindent
Intuitively, the addition operator starts by adding to $\o$ 
all the direct comparisons of $\p$ that are not involved in a cycle with it,
i.e., which are not under contention by the accrual of new preference information.
Then, at every further step $i>0$, the addition operator 
tries to add all comparisons on level $i$ of $\dir_\p$
that are involved in a cycle with $\o$:
if the resulting set of comparisons can be construed as a spo 
(by taking its transitive closure) the operation is successful, and the new comparisons are added;
if not, the addition operator does nothing.
Since the addition of new comparison follows the order $\le_\p$, this ensures
that better quality comparisons are considered before lower quality ones.

Note that this procedure guarantees that there are always \emph{some} comparisons in $\p\en\o$, 
i.e., we have that $\o\subseteq\p\en\o$, regardless of anything else.
Note, also, that the number of non-empty levels in $\dir_\p$ is finite and
the addition operation eventually reaches a fixed point, i.e., there exists $j\ge 0$ such that
$\add^i_{\le_\p}(\o)=\add^j_{\le_\p}(\o)$, for any $i\ge j$.
We denote by $\add_{\le_\p}^\ast(\o)$ the fixed point of this operator and take it as the defining expression of
a preference revision operator: if $\as$ is a preference assignment,
then the \emph{$\as$-induced preference revision operator $\en^\as$} 
is defined, for any $\p\in\CHN_V$
and $\o\in\SPO_V$, as:
$$
	\p\en^\as\o=\add^\ast_{\le_\p}(\o).
$$
Note that, by design, $\add^\ast_{\le_\p}(\o)\in\SPO_V$, i.e., the operator $\en$ is well defined.

\begin{figure}\centering
	\begin{tikzpicture}
	\node at (0,-1.5){$\p$};
	\node at (0,1.8)(1){$1$};
	\node at (0,0.9)(2){$2$};
	\node at (0,0)(3){$3$};
	\node at (0,-0.9)(4){$4$};
	\path[-latex](1)edge(2)(2)edge(3)(3)edge(4);
	\path[-latex,dotted](1)edge[bend right](3)(2)edge[bend right](4)(1)edge[bend right=40](4);
	
	\node at (1,-1.5){$\o$};
	\node at (1,1.8)(1){$1$};
	\node at (1,0.9)(2){$2$};
	\node at (1,0)(3){$3$};
	\node at (1,-0.9)(4){$4$};
	\path[-latex](3)edge[bend right](1);
	
	\node at (3,-1.5){$(\p\cup\o)^+$};
	\node at (3,1.8)(1){$1$};
	\node at (3,0.9)(2){$2$};
	\node at (3,0)(3){$3$};
	\node at (3,-0.9)(4){$4$};
	\path[-latex](1)edge(2)(2)edge(3)(3)edge(4)(3)edge[bend right](1);
	\path[-latex,dotted]
	(1)edge[bend right=60](3)(2)edge[bend right=50](4)
	(1)edge[bend right=70](4)(2)edge[bend left=30](1)
	(3)edge[bend left=30](2);
	
	\node at (5, -1.5){$\le_\p$};		
	\node at (5,-0.9){$(1,2)$};
	\node at (5,0){$(2,3)$, \sout{$(3,4)$}};
	\node at (6.5,-0.9){\tiny $\lvl^1_{\le_\p}(\p)$};
	\node at (6.5,0){\tiny $\lvl^2_{\le_\p}(\p)$};
	\end{tikzpicture}	
	\caption{
		Preference revision by adding direct comparisons from $\p$ to $\o$, using the preorder $\le_\p$.
		In $\le_\p$ lower means better; the comparison $(3,4)$ is ignored by the addition operator because it
		is not involved in a cycle with $\o$ (and is added at the beginning anyway).
	}
	\label{fig:preferences-4-items}
\end{figure}
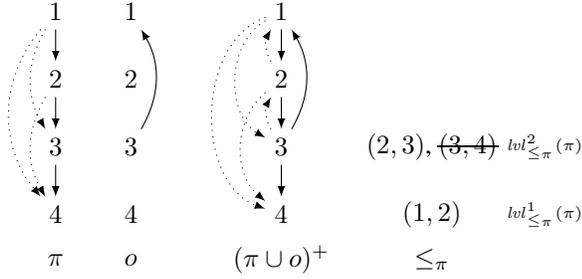

\begin{example}\label{ex:preferences-4-items}
	Consider initial preference order
	$\p=(1,2,3,4)$ and new information $\o=(3,1)$.
	We obtain that the direct comparisons of $\p$ are $\dir_\p=\{(1,2),(2,3),(3,4)\}$.
	Suppose, now, that there is a total preorder $\le_\p$ on $\dir_\p$ according to which
	$(1,2)<_\p (2,3)\approx_\p (3,4)$, as depicted in Figure \ref{fig:preferences-4-items}.
	To construct $\p\en\o$, the addition operator starts from
	$\add^0_{\le_\p}(\o)=(\{(3,1)\}\cup\{(1,4),(2,4),(3,4)\})^+$,
	i.e., $\o$ itself together with $\cycfree{\p}{\o}$, 
	the cycle-free part of $\p$ with respect to $\o$. 
	At the next step the addition operator tries to add $(1,2)$,
	which it can do successfully; at the next step it attempts to add $(2,3)$, which 
    creates a conflict with $(3,1)$ and $(1,2)$, added previously.
    After this there are no more comparisons to add.	
\end{example}

\section{Postulates for Preference Revision}\label{sec:postulates}
We show now that the procedure described in Section \ref{sec:semantics}
can be characterized with a set of AGM-like postulates that 
do not reference any concrete revision procedure and are, by themselves,
intuitive enough to provide reasonable constraints on any preference revision operator.

The first two postulates we consider apply to any chain $\p\in\CHN_V$, spo $\o\in\SPO_V$ 
and preference revision operator $\en\colon\CHN_V\times\SPO_V\rightarrow\SPO_V$, 
and are as follows:

\begin{description}
	\item[($\ppp{1}$)] $\pi\en\o\in\upper{\o}_\p$.
	\item[($\ppp{2}$)] $\cycfree{\p}{\o}\subseteq \p\en\o$.
\end{description}

\noindent
Postulates $\ppp{1-2}$ require the result to be formed by adding elements from $\p$ to 
the new information $\o$, and to be of a certain admissible type, i.e., an spo.
They are meant to capture preference revision in its most uncontroversial aspects,
yet they still require some careful unpacking.

Postulate $\ppp{1}$ states that $\p\en\o$ is a $\p$-completion of $\o$,
i.e., a preference order constructed only by adding direct comparisons from $\p$ to $\o$.
Unfolding its consequences, postulate $\ppp{1}$ ensures that:
\begin{description}
	\item[($i$)] $\p\en\o\in\SPO_V$, i.e., $\p\en\o$ is a chain on $V$, 
	\item[($ii$)] $\o\subseteq\p\en\o$, i.e., $\p\en\o$ contains all the information present in $\o$, and
	\item[($iii$)] $\p\en\o\subseteq(\p\cup\o)^+$, i.e., $\p\en\o$ is contained in the binary relation obtained 
	by simply adding $\o$ to $\p$, and adding all the comparisons inferred by transitivity. 
\end{description}

\noindent 
In terms of AGM propositional belief revision, 
postulate $\ppp{1}$ does the same duty as the \emph{Closure}, \emph{Success}, \emph{Inclusion} and \emph{Consistency}
postulates \cite{Hansson17,FermeH18}.
These postulate mandate that the revision result should be a propositional theory (i.e., have a required format), 
that the new information should be accepted, 
and that, unless the new information is inconsistent, the revision result should be consistent.

Given this observation, a question emerges as to why not use
conditions ($i$)-($iii$) as postulates instead of the proposed $\ppp{1}$.
The reason is that $\ppp{1}$ contains an element that lacks from conditions ($i$)-($iii$): 
what $\ppp{1}$ adds is the requirement that $\p\en\o$ is to be constructed using only direct comparisons 
of $\p$ (in addition to $\o$), and the reason why such a condition is desirable is to prevent $\p\en\o$ 
from having opinions on items over which no opinion had been expressed before revision. 
The issue is illustrated in Example \ref{ex:P1}.

\begin{example}\label{ex:P1}
	Consider preferences $\p$ and $\o$ as in 
	Example \ref{ex:preferences-3-items}, 
	and an additional spo $\p_4 = \{(3,1),(3,2)\}$.
	Note that $\p_4$ is such that $\o\subseteq\p_4\subseteq(\p\cup\o)^+$
	and therefore satisfies conditions ($i$)-($iii$) expressed above, 
	so that according to conditions ($i$)-($iii$) preference $\p_4$ is a viable 
	revision result.

	At the same time, we do not want to consider $\p_4$ as a potential candidate 
	for the revision result:
	the comparison $(3,2)$ occurs neither 
	in $\p$ nor in $\o$ as a direct comparison, 
	and there is reason to think that adding it would be unjustified:
	a rational preference revision operator should not be allowed to return $\p_4$
	when revising $\p$ by $\o$. 
	By contrast, when the comparison $(3,2)$ does occur, 
	e.g., in the desirable preference order $\p_1=(3,1,2)$,
	it occurs as the result of inference from $(3,1)$, 
	which is added from $\o$,
	and $(1,2)$, which is preserved from $\p$.
\end{example} 

\noindent
Postulate $\ppp{2}$ says that the cycle-free part of $\p$ with respect to 
$\o$ is to be preserved 
in $\p\en\o$, and is meant to preserve the parts of 
$(\p\cup\o)^+$ that are not up for dispute.
Note that in the case when $(\p\cup\o)^+$ 
does not contain a cycle then $\cycfree{\p}{\o}=(\p\cup\o)^+$,
and $\ppp{2}$ together with $\ppp{1}$ imply that $\p\en\o=(\p\cup\o)^+$: this is the case
when revision is easy, and nothing special needs to be done. 
Throughout all this, postulate $\ppp{2}$
serves the same function as the \emph{Vacuity} postulate in 
propositional revision~\cite{Hansson17,FermeH18}:
in the ideal case, when $\o$ can simply be added to $\p$, 
applying postulate $\ppp{2}$ results in 
the union of the two structures.

So far we have established that, if there is no conflict between $\p$ and $\o$,
i.e., no cycle arises by adding $\o$ to $\p$,
then we can simply add $\o$ to $\p$; and if there is a conflict, then $\en$ must choose
between the direct comparisons of $\p$ involved in the cycle.
This choice, however, must be coherent in a precise sense: we expect the choices 
to be indicative of an underlying preference over direct comparisons, 
which remains stable across different instances of revision.
This sense of coherence is illustrated by Example \ref{ex:choice}.

\begin{example}\label{ex:choice}
	Consider revising $\p=(1,2,3,4)$, 
    depicted in Figure \ref{fig:preferences-4-items}, by $\o_1=(4,1)$.
	SInce adding $(\p\cup\o)^+$ contains a cycle, revision requires a choice 
    between comparisons $(1,2)$, $(2,3)$ and $(3,4)$:
	assume $(1,2)$ is chosen, suggesting $(1,2)$ is better than $(2,3)$ and $(3,4)$. 
	Suppose, now, that we add $\o_2 = \{(3,4)\}$
    and revise by $(\o_1\cup\o_2)^+=\{(3,4),(4,1),(3,1)\}$: 
    another cycle is formed, and a choice is necessary, this time only between 
    $(1,2)$ and $(2,3)$. In accordance with the previous decision, $(1,2)$ should be chosen 
	here as well.
\end{example}

\noindent 
The choice behavior of a revision operator has to reflect an implicit preference 
order over the direct comparisons of $\p$,
and this is handled by the following postulates, 
meant to apply to any chain $\p\in\CHN_V$, spos $\o_1,\o_2\in\SPO_V$
such that $(\o_1\cup\o_2)^+\in\SPO_V$,
and a preference revision operator $\en$:

\begin{description}
	\item[($\ppp{3}$)] $\p\en(\o_1\cup\o_2)^+\subseteq((\p\en\o_1)\cup\o_2)^+$.
	\item[($\ppp{4}$)] If $((\p\en\o_1)\cup\o_2)^+\in\SPO_V$, then $((\p\en\o_1)\cup\o_2)^+\subseteq \p\en(\o_1\cup\o_2)^+$.
\end{description}

\noindent
There is a similarity between postulates $\ppp{3}$ and $\ppp{4}$ and 
the \emph{Superexpansion} and \emph{Subexpansion} postulates, respectively, 
from propositional belief revision~\cite{Hansson17,FermeH18},
which ensure that the choice between two options is stable and independent of alternatives 
not directly involved.
Postulates $\ppp{3-4}$ are meant to ensure the same here.
However, it turns out that in the context
of preference revision 
this happens only under a specific set of conditions, 
which we elaborate on in the following section.

\section{Coordination}\label{sec:coordination}
In this section we identify the precise conditions under which it makes sense to 
apply postulates $\ppp{3-4}$, presented in Section \ref{sec:postulates}.
Before doing so, we introduce some additional notation.

If $\o_1$ and $\o_2$ are spos,
we say that $\o_1$ and $\o_2$ are \emph{coordinated with respect to $\p$} 
if for any set $\d\subseteq\cyc{\p}{\o_1}$
such that for every direct comparison $(i,\tightplus{i}{1})\in\delta$, 
neither $(i,\tightplus{i}{1})$ nor $(\tightplus{i}{1},i)$ is in $(\o_1\cup\o_2)^+$, 
it holds that if $(\o_1\cup\d)^+\in\SPO_V$,
then $((\o_1\cup\o_2)^+\cup\d)^+\in\SPO_V$.
In other words, if $\p$ and $\o_1$ form a cycle
and we want to add $\o_2$ as well,
then we direct our attention to the direct comparisons in $\p$
that are not directly ruled out by $(\o_1\cup\o_2)^+$,
i.e., such that neither these comparisons nor their inverses are 
contained in $(\o_1\cup\o_2)^+$.
The property of coordination says that 
if we can consistently add some of these 
comparisons to $\o_1$,
then it must be the case that we can also add them to $(\o_1\cup\o_2)^+$. 
Intuitively, coordination means that adding
extra information $\o_2$ does not step on $\o_1$'s toes,
by rendering unviable any
comparisons that were previously viable.
The following example makes this clearer.

\begin{figure}\centering
	\begin{tikzpicture}
	\node at (0,-1.5){$\p$};
	\node at (0,1.8)(1){$1$};
	\node at (0,0.9)(2){$2$};
	\node at (0,0)(3){$3$};
	\node at (0,-0.9)(4){$4$};
	\path[-latex](1)edge(2)(2)edge(3)(3)edge(4);
	\path[-latex,dotted](1)edge[bend right](3)(2)edge[bend right](4)(1)edge[bend right=40](4);
	
	\node at (1.3,-1.5){$\le_\p$};		
	\node at (1.3,-0.9){$(1,2)$};
	\node at (1.3,0){$(2,3)$};
	\node at (1.3,0.9){$(3,4)$};
	
	\node at (2.9,-1.5){$\o_1$};
	\node at (2.9,1.8)(1){$1$};
	\node at (2.9,0.9)(2){$2$};
	\node at (2.9,0)(3){$3$};
	\node at (2.9,-0.9)(4){$4$};
	\path[-latex](4)edge[bend right](1);
	
	\node at (3.7,-1.5){$\o_2$};
	\node at (3.7,1.8)(1){$1$};
	\node at (3.7,0.9)(2){$2$};
	\node at (3.7,0)(3){$3$};
	\node at (3.7,-0.9)(4){$4$};
	\path[-latex](3)edge[bend right](1);
	
	\node at (5,-1.5){\tiny $\p\en(\o_1\cup\o_2)^+$};
	\node at (5,1.8)(1){$1$};
	\node at (5,0.9)(2){$2$};
	\node at (5,0)(3){$3$};
	\node at (5,-0.9)(4){$4$};
	\path[-latex](4)edge[bend right](1);
	\path[-latex](3)edge[bend right](1);
	\path[-latex](1)edge(2)(3)edge[dotted,bend left](2);
	\path[-latex](3)edge(4);
	
	\node at (6.9,-1.5){\tiny $((\p\en\o_1)\cup\o_2)^+$};
	\node at (6.9,1.8)(1){$1$};
	\node at (6.9,0.9)(2){$2$};
	\node at (6.9,0)(3){$3$};
	\node at (6.9,-0.9)(4){$4$};
	\path[-latex](4)edge[bend right](1)(1)edge(2)(2)edge(3)(1)edge[dotted, bend right](3)(4)edge[dotted](3);
	\path[-latex](3)edge[bend right](1);
	\end{tikzpicture}	
	\caption{
		Postulates $\ppp{3-4}$ are satisfied only if $\o_1$ and $\o_2$ are coordinated with respect to $\p$.
		}
	\label{fig:non-coordination}
\end{figure}
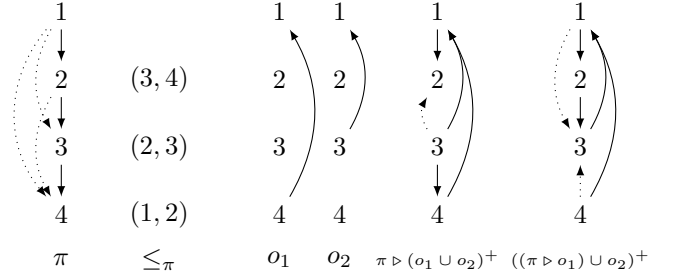

\begin{example}\label{ex:coordination}
	Take $\p=(1,2,3,4)$ and $\o_1=(4,1)$, $\o_2=(3,1)$.
	The direct comparisons of $\p$ that are involved in a cycle
	with $\o_1$ are $\cyc{\p}{\o_1}=\{(1,2),(2,3),(3,4)\}$,
	so that revision by $\o_1$ requires making a choice between 
	these comparisons.
	This choice, we expect, is done on the basis of some implicit preference over the comparisons,
	which guides revision even when we add additional information in the form of $\o_2$.
	Notice, now, that neither of $(1,2)$, $(2,3)$ and $(3,4)$ is individually ruled out by 
	$(\o_1\cup\o_2)^+$: we have, for instance, that $(1,2)\notin(\o_1\cup\o_2)^+$
	and $(2,1)\notin(\o_1\cup\o_2)^+$; the same holds for $(2,3)$ and $(3,4)$.
	The significance of this is that adding $\o_2$ to $\o_1$ does not alter the menu:
	the choice is still one over comparisons
	$(1,2)$, $(2,3)$ and $(3,4)$.

	The problem, however, is that whereas with $\o_1$ the choice is relatively unconstrained,
	meaning we can choose any proper subset of $\{(1,2), (2,3), (3,4)\}$ to add to $(4,1)$, 
	adding the additional comparison $(3,1)$ complicates things.
	To see how, consider the set of comparisons
	$\d=\{(1,2),(2,3)\}$. These comparisons can be consistently added to $\o_1$,
	i.e., $(\o_1\cup\d)^+\in\SPO_V$, but
	not to $(\o_1\cup\o_2)^+$,
	i.e.,
	$((\o_1\cup\o_2)^+\cup\d)^+\notin\SPO_V$.
	According to our definition, this implies that $\o_1$ and $\o_2$ are not coordinated with respect to $\p$.
	Thus, whereas with $\o_1$ can be augmented with both $(1,2)$ and $(2,3)$,
	$\o_1$ and $\o_2$ do not allow adding both comparisons \emph{together}.
	This, then, has a knock-down effect in that it makes it possible to add comparison $(3,4)$, 
	irrespective of where it is in the preorder on comparisons.
	
	In such a situation, then, the specific details of how the choice problem is constructed 
	makes the position of $(3,4)$ in the overall preference order over comparisons irrelevant. 
	Consequently, expecting our axioms to take the preference order into account will land us into trouble.
	To see this, onsider preorder $\le_\p$ in Figure \ref{fig:non-coordination}, 
	where $(3,4)$ is the least preferred comparison,
	and the revision operator $\en$ induced by it.
	We have that $(3,4)\in\p\en(\o_1\cup\o_2)^+$, but $(3,4)\notin((\p\en\o_1)\cup\o_2)^+$,
	i.e., postulate $\ppp{3}$ is not satisfied.

	This fact is related with the lack of coordination between $\o_1$ and $\o_2$, 
	as the addition of $\o_2$ tampers with the choice problem:
	though we can still add either one of the three comparisons, as mentioned above,
	we cannot add $(1,2)$ and $(2,3)$
	together anymore, which in turn means that $(3,4)$ can be added regardless of its position in $\le_\p$ the preorder.
\end{example}

\noindent
Example \ref{ex:coordination} is a case in which lack of coordination 
creates a situation where postulate $\ppp{3}$ is not satisfied.
We do not mean to imply, however, that there is anything wrong with postulate $\ppp{3}$, 
or with uncoordinated preference information.
Rather, we take the moral to be that we need postulates tailored to cases that do \emph{not} look like the one 
in Example \ref{ex:coordination}, 
in which preference information over the direct comparisons 
is rendered unusable by the overriding structural constraints of working with preference orders.

In other words, we want the behavior of a revision operator to reflect the preference information 
over the direct comparisons:
however, the requirement of transitivity means that, in the interest of consistency, 
we sometimes have to add comparisons that were not explicitly chosen, 
and this can interfere with the preference information over the comparisons of $\p$.
Thus, the significance of coordination, as the following theorem shows,
is that it is needed in order for postulates $\ppp{3-4}$ to be effective at 
ensuring that choice across different types of incoming preferences
is coherent.

\begin{theorem}\label{thm:P34-coordination}
	If $\as\colon\CHN_V\rightarrow\TPRE_{V\times V}$ 
	is a preference assignment and $\en^\as$ 
	is the $\as$-induced revision operator,
	then, 
	$\en^\as$ satisfies postulates $\ppp{3-4}$
	if and only if, 
	for any chain $\p\in\CHN_V$ and spos $\o_1,\o_2\in\SPO_V$, 
	it holds that $\o_1$ and $\o_2$ are coordinated with respect to $\p$.
\end{theorem}
\begin{proof}
	(``$\Leftarrow$'')
	Take $\o_1,\o_2\in\SPO_V$ that are coordinated with respect to $\p$.
	We will show that, for any preorder $\le_\p$ on $\dir_\p$,
	the $\as$-induced revision operator $\en^\as$ satisfies postulates $\ppp{3-4}$.
	Since 
	$\en^\as$ satisfies postulates $\ppp{3-4}$ trivially
	if $(\p\cup\o_1)^+\in\SPO_V$,
	we look at the case when $\cyc{\p}{\o_1}\neq\emptyset$,
	i.e., when $(\p\cup\o_1)^+$ contains a cycle.
	
	For postulate $\ppp{3}$, assume there is a 
	comparison $c^\star\in\add^\ast_{\le_\p}(\o_1\cup\o_2)^+$
	such that $c^\star\notin(\add^\ast_{\le_\p}(\o_1)\cup\o_2)^+$.
	If $c^\star\in(\o_1\cup\o_2)^+$ then a contradiction follows immediately.
	We thus have to look 
	at the case when $c^\star\notin(\o_1\cup\o_2)^+$, which contains two subcases 
	of its own.

	\emph{Case 1}. 
	If $c^\star\in\dir_\p$, then by our assumption 
	we have that $c^\star\in\cyc{\p}{\o_1}$, i.e., 
	$c^\star$ is involved in some cycle with $\o_1$.
	From $c^\star\notin\add^\ast_{\le_\p}(\o_1)$ 
	we infer that there must be a set $\d\subseteq\dir_\p$
	of direct comparisons of $\p$ 
	that precede $c^\star$ in $\le_\p$, are added to $\o_1$ before it, 
	and prevent $c^\star$ itself from being added.
	In particular, this means that $(\o_1\cup\d)^+\in\SPO_V$, 
	but $((\o_1\cup\d)^+\cup\{c^\star\})^+\notin\SPO_V$. 
	At the same time, we know that $c^\star\in\add^\ast_{\le_\p}(\o_1\cup\o_2)^+$,
	i.e., $c^\star$ can be consistently added to $(\o_1\cup\o_2)^+$.
	Note that this happens after all the comparisons in $\d$, which precede it in $\le_\p$,
	have been considered as well.
	This implies that not all of the comparisons in $\d$ can be added to $(\o_1\cup\o_2)^+$,
	since if they could, then the cycle formed with $\o_1$, $\d$ and $c^\star$ would be reproduced here
	as well. 
	If not all of the comparisons in $\d$ can be added to $(\o_1\cup\o_2)^+$,
	this must be because $((\o_1\cup\o_2)^+\cup\d)^+$ contains a cycle,
	i.e., $((\o_1\cup\o_2)^+\cup\d)^+\notin\SPO_V$.
	This now contradicts the fact that $\o_1$ and $\o_2$ are coordinated with respect to $\p$.
	
	\emph{Case 2}. 
	If $c^\star$ is not a direct comparison of $\p$, then it is inferred 
	by transitivity using at least one direct comparison of $\p$ added previously. 
	We apply the reasoning in Case 1 to these direct comparisons to show that they are in 
	$(\add^\ast_{\le_\p}(\o_1)\cup\o_2)^+$, which implies 
		 that 
		$c^\star\in(\add^\ast_{\le_\p}(\o_1)\cup\o_2)^+$ as well.	
	
	For postulate $\ppp{4}$, take $c^\star\in(\add^\ast_{\le_\p}(\o_1)\cup\o_2)^+$
	and assume $c^\star\notin\add^\ast_{\le_\p}(\o_1\cup\o_2)^+$.
	As before, the non-obvious case is when $c^\star\notin(\o_1\cup\o_2)^+$.
	If $c^\star\in\dir_\p$, 
	then from the assumption that $c^\star\notin\add^\ast_{\le_\p}(\o_1\cup\o_2)^+$
	we conclude that there is a set $\d\subseteq\cyc{\p}{\o_1}$ of comparisons 
	that precede $c^\star$ in $\le_\p$,
	are added to $(\o_1\cup\o_2)^+$ before it and, in concert with $(\o_1\cup\o_2)^+$, 
	block $c^\star$ from being added, i.e., 
	such that:
	$$((\o_1\cup\o_2)^+\cup\d)^+\in\SPO_V,$$ 
	but	$((\o_1\cup\o_2)^+\cup\d')^+\notin\SPO_V$,
	where $\d'=\d\cup\{c_\star\}$.
	From the second to last result 
	we infer that $\d$ can be added consistently to $(\o_1\cup\o_2)^+$
	and, since we have that $c^\star\in(\add^\ast_{\le_\p}(\o_1)\cup\o_2)^+$ as well,
	we obtain that 
	and $c^\star$ can be added consistently to $\o_1$.
	In other words, it holds that 
	$(\o_1\cup\d')^+\in\SPO_V$.
	Together with the previous result this contradicts the fact that $o_1$ and $\o_2$ are coordinated 
	with respect to $\p$.
	
	The case when $c^\star\notin(\o_1\cup\o_2)^+$ is 
	treated analogously as for postulate $\ppp{3}$.

	(``$\Rightarrow$'')
	Assume that there are $\o_1,\o_2\in\SPO_V$ not coordinated with respect to $\p$,
	i.e., there exists a set $\delta\subseteq\cyc{\p}{\o_1}$ of direct comparisons of $\p$
	that are involved in a cycle with $\o_1$ and are such that
	$(\o_1\cup\d)^+\in\SPO_V$ and $((\o_1\cup\o_2)^+\cup\d)^+\notin\SPO_V$.
	Additionally, we have that neither of the comparisons in $\delta$, 
	or their inverses, are in $(\o_1\cup\o_2)^+$.
	We infer that there must exist a comparison $c^\star\in(\cyc{\p}{\o_1}\setminus\d)$
	that completes the cycle. 
	We will show that there exists a preorder $\le_\p$ 
	such that the revision operator	induced by it does not satisfy $\ppp{3}$.
	Take a preorder $\le_\p$ on $\dir_\p$ that arranges the elements 
	of $\d$ in a linear order at the bottom of $\le_\p$, 
	i.e., such that $c_j<_\p c_l$, for any $c_j\in\d$ and $c_l\notin\d$,
	and $c^\star$ the maximal element in $\le_\p$,
	i.e., $c_j<_\p c^\star$, for any $c_j\in\d$.
	This implies, in particular, that $c_j<_\p c^\star$, for any $c_j\in\d$.
	
	Note, now, that $c^\star\in\add_{\le_\p}^\ast(\o_1\cup\o_2)^+$: 
	this is because, by assumption, not all of the comparisons in 
	$\d$ can be added to $(\o_1\cup\o_2)^+$,
	and this makes it possible for $c^\star$ to be added.
	On the other hand, $c^\star\notin(\add^\ast_{\le_\p}(\o_1)\cup\o_2)^+$:
	this is because here we can, again by assumption, consistently add $\d$ to $\o_1$ and,
	since $c^\star$ is the last in line to be added, the inevitability
	of creating a cycle with $\d$ and the rest of the comparisons of $\o_1$
	makes it impossible to do so consistently. 
	We obtain that $c^\star\in\add_{\le_\p}^\ast(\o_1\cup\o_2)^+$ but 
	$c^\star\notin(\add^\ast_{\le_\p}(\o_1)\cup\o_2)^+$, i.e., postulate $\ppp{3}$ is not satisfied.
	Concurrently, there will be a comparison in $\d$ that occurs in $(\add^\ast_{\le_\p}(\o_1)\cup\o_2)^+$
	that does not make it into $\add^\ast_{\le_\p}(\o_1\cup\o_2)^+$, 
	showing that $\ppp{4}$ is not satisfied either.
\end{proof}

\noindent
Theorem \ref{thm:P34-coordination} shows that coordination is needed in order to make
sure that postulates $\ppp{3-4}$ work, and we will henceforth assume that $\o_1$ and $\o_2$ 
are coordinated with respect to $\p$ whenever we apply these postulates.

\section{Characterizing Preference Revision as Choice Over Comparisons}\label{sec:representation}
We show now that the procedure described in 
Section \ref{sec:semantics} is characterized
by the postulates introduced 
in Section \ref{sec:postulates}, under the restrictions
established through Theorem \ref{thm:P34-coordination}.
Theorem \ref{thm:repr-ties-lr} shows that the procedure in 
Section \ref{sec:semantics} 
yields preference revision operators that satisfy postulates $\ppp{1-4}$.

\begin{theorem}\label{thm:repr-ties-lr}
	If $\as\colon\CHN_V\rightarrow\TPRE_{V\times V}$ is a preference assignment,
	then the revision operator $\en^\as$ induced by it satisfies postulates $\ppp{1-4}$,
	for any $\p\in\CHN_V$ and $\o,\o_1,\o_2\in\SPO_V$ such that $\o_1$, $\o_2$ are coordinated 
	with respect to $\p$.
\end{theorem}
\begin{proof}
	Satisfaction of postulates $\ppp{1-2}$ is straightforward.
	For $\ppp{1}$, since at every step $\add^i_{\le_\p}$ 
	selects some direct comparisons in $\p$ to add to $\o$,
	the end result satisfies the condition for being in $\upper{\o}_\p$.
	For $\ppp{2}$, note that 
	$(\p\cup\o)^+\subseteq\add^0_{\le_\p}(\o)\subseteq\add_{\le_\p}^\ast(\o)$.
	Since $\o_1$ and $\o_2$ are assumed to be coordinated with respect to $\p$, 
	satisfaction of postulates $\ppp{3-4}$ 
	is guaranteed by Theorem \ref{thm:P34-coordination}.	
\end{proof}

\noindent
For the converse, we want to show that any preference revision operator satisfying $\ppp{1-4}$
can be rationalized using a preference assignment.

To that end, we will construct the preorder $\le_\p$ from binary comparisons,
but we must first figure out how to compare two direct comparisons $\drc{k}$ and $\drc{l}$.
This is done by creating a situation where we cannot add both and hence one has to be given up.
We will use a special type of preference order to induce a choice between these comparisons.
If $\p\in\CHN_V$ is a chain and $\drc{k},\drc{l}\in\dir_\p$ are direct comparisons of $\p$,
the \emph{choice inducing preference $\o_{k,l}$ for $\drc{k}$ and $\drc{l}$} is
defined as $\o_{k,l}=\{(\tightplus{k}{1},l),(l{+}1, k)\}$.
The following example illustrates this notion. 

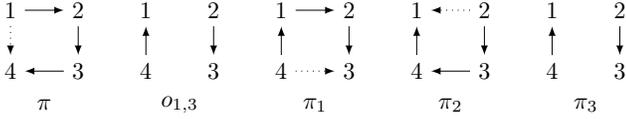
\begin{figure}\centering
	\begin{tikzpicture}
	\scalebox{0.9}{
		\node at (0.5,-0.5){$\p$};
		\node at (0,0.9)(1){$1$};
		\node at (1,0.9)(2){$2$};
		\node at (1,0)(3){$3$};
		\node at (0,0)(4){$4$};
		\path[-latex] (1)edge(2)(3)edge(4)(2)edge(3)(1)edge[dotted](4);
		
		\node at (2.5,-0.5){$\o_{1,3}$};
		\node at (2,0.9)(1){$1$};
		\node at (3,0.9)(2){$2$};
		\node at (3,0)(3){$3$};
		\node at (2,0)(4){$4$};
		\path[-latex] (2)edge(3)(4)edge(1);
		
		\node at (4.5,-0.5){$\p_1$};
		\node at (4,0.9)(1){$1$};
		\node at (5,0.9)(2){$2$};
		\node at (5,0)(3){$3$};
		\node at (4,0)(4){$4$};
		\path[-latex] (1)edge(2)(2)edge(3)(4)edge(1)(4)edge[dotted](3);
		
		\node at (6.5,-0.5){$\p_2$};
		\node at (6,0.9)(1){$1$};
		\node at (7,0.9)(2){$2$};
		\node at (7,0)(3){$3$};
		\node at (6,0)(4){$4$};
		\path[-latex] (4)edge(1)(2)edge(3)(3)edge(4)(2)edge[dotted](1);
		
		\node at (8.5,-0.5){$\p_3$};
		\node at (8,0.9)(1){$1$};
		\node at (9,0.9)(2){$2$};
		\node at (9,0)(3){$3$};
		\node at (8,0)(4){$4$};
		\path[-latex] (4)edge(1)(2)edge(3);
	}
	\end{tikzpicture}
	\caption{
		Revision of $\p$ by $o_{1,3}$ forces a choice between direct comparisons
		$(1,2)$ and $(3,4)$: since keeping both $(1,2)$ and $(3,4)$ is not possible, 
		at least one of them, potentially both, must be discarded.
		Depending on the choice made, possible results are $\p_1$, $\p_2$ and $\p_3$.
	}
	\label{fig:choice-inducing-preference}
\end{figure}

\begin{example}\label{ex:choice-inducing-preference}
	To induce a choice between direct comparisons 
	$(1,2)$ and $(3,4)$ in Figure \ref{fig:choice-inducing-preference}, revise by 
	$\o_{1,3}=\{(2,3),(4,1)\}$.
	Note that effectiveness of this maneuver hinges on the choice 
	being confined to the direct comparisons of $\p$:
	if inferred comparisons were allowed to be part of the choice, 
	$\o_{1,3}$ loses its power to discriminate between $(1,2)$ and $(3,4)$:
	if, for instance, $(1,3)$ and $(2,4)$ are chosen, then $(2,1)$ and $(4,3)$ 
	have to be inferred, leaving no space for a choice between $(1,2)$ and $(3,4)$,
	i.e., $\o_{1,3}$ would tell us nothing about the implicit preference between $(1,2)$ and $(3,4)$.
	We can also see that comparison of $(1,2)$ and $(2,3)$ is done by revising by $(3,1)$.	
\end{example}

\noindent 
Conversely, if $\drc{k},\drc{l}\in\dir_\p$
and $\en$ is a preference revision operator, then the 
\emph{revealed order $\le^\en_\p$ between $\drc{k}$ and $\drc{l}$} is defined as:
$$
	\drc{k}\le^\en_\p \drc{l}~\textnormal{if}~\drc{l}\notin\p\en\o_{k,l}.
$$
Intuitively, $\drc{l}$ being discarded from $\p\en\o_{k,l}$ signals 
that it is considered less important than $\drc{k}$. 

The primary question at this point is whether the revealed preference relation $\le^\en_\p$, 
as defined above, is transitive. We show next that the answer is yes.

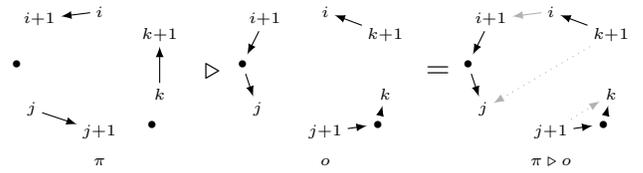
\begin{figure}\centering
	\begin{tikzpicture}\tiny
	\node at (0,-1.3){$\p$};
	\node at (0,0.7)(1){$i$};
	\node at (-0.8, 0.6)(2){$i{+}1$};
	\node at (-1.1,0)(3){$\bullet$};
	\node at (-0.9, -0.6)(4){$j$};
	\node at (0,-0.9)(5){$j{+}1$};
	\node at (0.7,-0.8)(6){$\bullet$};
	\node at (0.8,-0.4)(7){$k$};
	\node at (0.8,0.4)(8){$\tightplus{k}{1}$};
	\path[-latex] (1)edge(2)(4)edge(5)(7)edge(8);
	
	\node at (1.5,-0.1){\Large $\en$};
	
	\node at (3,-1.3){$\o$};
	\node at (3,0.7)(1){$i$};
	\node at (2.2, 0.6)(2){$i{+}1$};
	\node at (1.9,0)(3){$\bullet$};
	\node at (2.1, -0.6)(4){$j$};
	\node at (3,-0.9)(5){$j{+}1$};
	\node at (3.7,-0.8)(6){$\bullet$};
	\node at (3.8,-0.4)(7){$k$};
	\node at (3.8,0.4)(8){$\tightplus{k}{1}$};
	\path[-latex] (2)edge(3)(3)edge(4)(5)edge(6)(6)edge(7)(8)edge(1);
	
	\node at (4.5,-0.1){\Large $=$};
	
	\node at (6,-1.3){$\p\en\o$};
	\node at (6,0.7)(1){$i$};
	\node at (5.2, 0.6)(2){$i{+}1$};
	\node at (4.9,0)(3){$\bullet$};
	\node at (5.1, -0.6)(4){$j$};
	\node at (6,-0.9)(5){$j{+}1$};
	\node at (6.7,-0.8)(6){$\bullet$};
	\node at (6.8,-0.4)(7){$k$};
	\node at (6.8,0.4)(8){$\tightplus{k}{1}$};
	\path[-latex] (2)edge(3)(3)edge(4)(5)edge(6)(6)edge(7);
	\path[-latex] (1)edge[gray!60](2)(8)edge(1);
	\path[-latex] (8)edge[dotted,gray!60](4)(5)edge[dotted,gray!60](7);
	\end{tikzpicture}
	\caption{
		To show that $\le_\p^\en$ is transitive,
		we show first that $\drc{k}\notin\p\en\o$.
		Bullets indicate other potential items in $\p$;
		faded arrows indicate comparisons that may not be in $\p\en\o$,
		but can be consistently added to it.
	}
	\label{fig:acyclic-proof}
\end{figure}

\begin{lemma}\label{lem:revealed-preference-relation-transitive}
	If $\en$ satisfies postulates $\ppp{1-4}$, 
	then the revealed preference relation $\le_\p^\en$	
	is transitive.
\end{lemma}
\begin{proof}	
	Take $\p\in\CHN_V$ and $\drc{i}, \drc{j}, \drc{k}\in\dir_\p$
	such that $\drc{i}\le^\en_\p\drc{j}\le^\en_\p\drc{k}$
	(we can assume that $i<j<k$).
	To show that $\drc{i}\le^\en_\p\drc{j}$,
	take $\o\in\SPO_V$ that contains all direct comparisons in $\p$
	up to $k$, except $\drc{i}$, $\drc{j}$ and $\drc{k}$, 
	plus the comparison $(\tightplus{k}{1},i)$. 
	In other words,
	$\o$ is such that if $\drc{i}$, $\drc{j}$ and $\drc{k}$ were added to it,
	a cycle would form.
	The first step involves showing that $\drc{k}\notin\p\en\o$.	
	To see why this is the case, 
	note first that, 
	by design, not all of $\drc{i}$, $\drc{j}$ and $\drc{k}$ can be in $\p\en\o$,
	i.e., at least one of them must be left out. 
	We now do a case analysis
	to show that,
	either way, $\drc{k}$ ends up being left out.
	
	\emph{Case 1}. If $\drc{k}\notin\p\en\o$, the conclusion is immediate.
	
	\emph{Case 2}. If $\drc{j}\notin\p\en\o$, then we can safely add $\drc{i}$ to $\p\en\o$:
	this is because the inference of 
	the opposite comparison, i.e., 
	$(i{+}1,i)$, can be done
	only by adding all comparisons on the path from $i{+}1$ to $i$, and the absence
	of $\drc{j}$ means this inference is blocked. Using $\ppp{3-4}$ we can now conclude that
	$((\p\en\o)\cup\{\drc{i}\})^+=\p\en(\o\cup\{\drc{i}\})^+$
	(see Figure \ref{fig:acyclic-proof}).
	Note, we can separate $\o\cup\{\drc{i}\}$ into $\o_{j,k}=\{(k{+}1,j), (j{+}1,k)\}$ and 
	all the comparisons on the path from $k{+}1$ to $j$, plus the comparisons on the path
	from $j{+}1$ to $k$. Call this latter preference $\o'$.
	We thus have that $(\o\cup\{\drc{i}\})^+=(\o_{j,k}\cup\o')^+$
	and, applying $\ppp{3}$, we obtain that:
	$$
	\p\en(\o\cup\{\drc{i}\})^+ = \p\en(\o_{j,k}\cup\o')^+\subseteq((\p\en\o_{j,k})\cup\o')^+.
	$$		
	Since, by definition, $\drc{k}\notin\p\en\o_{j,k}$ and $\drc{k}\notin\o'$,
	It follows that: 
	$$\drc{k}\notin\p\en(\o\cup\{\drc{i}\})^+,$$
	then: 
	$$\drc{k}\notin((\p\en\o)\cup\{\drc{i}\})^+,$$
	and, finally, that: 
	$$\drc{k}\notin\p\en\o.$$
	
	\emph{Case 3}. 
	If $\drc{i}\notin\p\en\o$, then we can safely add $\drc{k}$ to $\p\en\o$
	and, by reasoning similar to above, show that 
	$\drc{j}\notin\p\en\o$. Here we invoke Case 2.
	
	With the fact that $\drc{k}\notin\p\en\o$ in hand, 
	we can add $\drc{j}$ to $\p\en\o$ (by reasoning similar to above),
	because the path from $j{+1}$ to $j$ in $\p\en\o$ is blocked by the absence of $\drc{k}$.
	Using postulates $\ppp{3-4}$, we conclude that:
		\begin{align*}
			((\p\en\o)\cup\{\drc{j}\})^+ &=\p\en(\o\cup\{\drc{j}\})^+\\
			&= \p\en(\{(i{+}1,\dots,k),(k{+1},i)\})^+\\
			&=((\p\en\o_{i,k})\cup\{(i{+}1,\dots,k)\})^+.
		\end{align*}
	Since $\drc{k}\notin((\p\en\o)\cup\{\drc{j}\})^+$, we conclude that 
	$\drc{k}\notin\p\en\o_{i,k}$,
	which implies that $\drc{i}\le^\en_\p\drc{k}$.
\end{proof}

\noindent 
Lemma \ref{lem:revealed-preference-relation-transitive} 
is crucial for the following 
representation result, 
as it indicates that we can identify the revealed preference relation 
with the underlying preference over direct comparisons of $\p$ driving revision.

\begin{theorem}\label{thm:repr-ties-rl}
	If $\en$ is a revision operator satisfying postulates $\ppp{1-4}$,
	for any $\p\in\CHN_V$ and $\o,\o_1,\o_2\in\SPO_V$ such that $\o_1$, $\o_2$ are coordinated 
	with respect to $\p$,
	then there exists a preference assignment $\as$ 
	such that $\en$ is the $\as$-induced revision operator.
\end{theorem}
\begin{proof}
	For any $\p\in\CHN_V$, take $\le_\p$ to be the revealed preference relation $\le^\en_\p$.
	By Lemma \ref{lem:revealed-preference-relation-transitive}, we know that 
	$\le_\p$ is transitive, 
	so the only thing left to is show is that $\p\en\o = \add^\ast_{\le_\p}(\o)$.
	We do this in two steps.
	
	(``$\subseteq$'') 
	For one direction, 
	Take $(j,k)\in\p\en\o$ and suppose 
	$(j,k)\notin\add^\ast_{\le_\p}(\o)$.
	Clearly, it cannot be the case that $(j,k)\in\o$, so we conclude that $(j,k)$
	is either a direct comparison of $\p$, or is inferred by transitivity using direct comparisons in $\p$
	and $\o$. 	
	
	\emph{Case 1}. If $(j,k)\in\dir_\p$, then we can write $(j,k)$ as $\drc{j}$,
	Suppose that $\drc{j}$ is on level $i$ of $\dir_\p$: this means that if $\drc{j}$ does not get added to 
	$\add^\ast_{\le_\p}(\o)$ at step $i$, then, since it cannot be inferred by transitivity,
	it does not get added at all. The fact that $\drc{j}\notin\add^\ast_{\le_\p}(\o)$ thus means that
	$\drc{j}$ forms a cycle with some comparisons in $\o$ and comparisons in $\p$ on levels $l\le i$.
	First, note that $\drc{j}$ cannot form a cycle with elements of $\o$ only, since that would imply
	that $(j{+}1,j)\in\o$ and that would exclude the possibility that $\drc{j}\in\p\en\o$.
	Thus, at least one other comparison in the cycle must come from $\p$.
	We can state, now, that, since $(j{+}1,j)\in\p\en\o$, then at least one of these comparisons 
	must be absent in $\p\en\o$,
	i.e., there exists a direct comparison $\drc{k}\in\dir_\p$ such that $\drc{k}\in\lvl^j_{\le_\p}(\p)$,
	for some $j\le i$, 
	$\drc{k}\notin\p\en\o$ and $\drc{j}$, $\drc{k}$, plus some other comparisons in $\o$ and $\p$ form a cycle.
	This means that it is safe to add $\o'$ to $\p\en\o$, where $\o'$ contains all comparisons on the path from $\tightplus{k}{1}$ to $j$,	plus the comparison on the path from $j{+}1$ to $k$.
	We can rewrite $\o'$ by separating out $(\tightplus{k}{1},j)$ and $(j{+}1,k)$, i.e., $\o'=(\o_{j,k}\cup\o')^+$.
	Applying postulates $\ppp{3-4}$, we now get that
	\begin{align*}
	((\p\en\o)\cup\o')^+ &= \p\en(\o\cup\o')^+\\
					   &= \p\en(\o_{j,k}\cup\o')^+\\
					   &\subseteq ((\p\en\o_{j,k})\cup\o')^+.						   
	\end{align*}
	Using the assumption that $\drc{j}\in\p\en\o$ and the fact that $\drc{j}\notin{\o'}$, 
	we can thus infer that $\drc{j}\in\p\en\o_{j,k}$.
	This, in turn, implies that $\drc{j}<_\p\drc{k}$ and hence $\drc{j}$ belongs to 
	a lower level of $\dir_\p$ than $\drc{k}$: but this contradicts the conclusion drawn earlier
	that $\drc{k}$ belongs to a level $l\le i$, where $i$ is the level of $\drc{j}$. 

	\emph{Case 2}. 
	If $(j,k)$ is not a direct comparison of $\p$, then it is inferred from some direct
	comparisons of $\p$ that end up in $\p\en\o$,
	together with comparisons in $\o$. We can now apply the reasoning from Case 1
	to the direct comparisons of $\p$ that go into inferring $(j,k)$, to show that they must be in 
	$\add^\ast_{\le_\p}(\o)$. 
	This, in turn, implies that $(j,k)$ will be in $\add^\ast_{\le_\p}(\o)$
	as well.
	
	The reasoning for the other direction is similar.
\end{proof}

\noindent
Theorems \ref{thm:repr-ties-lr} and \ref{thm:repr-ties-rl} 
describe preference revision operators that 
rely on total preorders $\le_\p$ on $\dir_\p$, where a tie between two direct comparisons 
means that if they cannot both be added, then they are both passed over.
We can eliminate this indeciseveness by using \emph{linear} orders on $\dir_\p$ instead of
preorders: this ensures that any two direct comparisons of $\p$ can be clearly ranked with respect
to each other, and that a revision operator is always in a position to choose among them.
On the postulate site, linear orders can be characterized by tightening the notion
of a $\p$-completion and, with it, postulate $\ppp{1}$.
Thus, a \emph{decisive $\p$-completion of $\o$} is
defined as: 
$$
\upperd{\o}_\p = \{(\o\cup\d)^+\in\SPO_V\mid\emptyset\subset\d\subseteq\dir_\p\}.
$$
Changing the format of the revision output requires changing the postulate that speaks about this format as well.
The decisive version of $\ppp{1}$ is then written, for any $\p\in\CHN_V$ and $\o\in\SPO_V$, as:
\begin{description}
	\item[($\ppp{\mathsf{D}}$)] $\p\en\o\in\upperd{\o}_\p$. 
\end{description}

A \emph{decisive preference assignment $\as$} is a 
function $\as\colon\CHN_V\rightarrow\CHN_{V\times V}$ mapping every $\p\in\CHN_V$ to a linear preorder $<_\p$ on $\dir_\p$. We can now show the following result.

\begin{theorem}\label{thm:repr-decisive}
	A revision operator $\en$ satisfies postulates $\ppp{\mathsf{D}}$ and $\ppp{2-4}$ 
	if and only if
	there exists a decisive preference assignment $a$ such that, for any $\p\in\CHN_V$
	and $\o,\o_1,\o_2\in\SPO_V$ such that $\o_1$, $\o_2$ are coordinated 
	with respect to $\p$, $\en$ is the $\as$-induced preference revision operator.
\end{theorem}
\begin{proof}
	The proofs for Theorems \ref{thm:repr-ties-lr} and \ref{thm:repr-ties-rl} 
	work here with minimal adjustments.
	Note that when choosing between two direct comparisons, postulate $\ppp{\mathsf{D}}$ does not allow $\en$ 
	to be indifferent anymore. This means that the revealed preference relation on $\dir_\p$ ends up being linear.
\end{proof}

\noindent
We can see, thus, that what seems like a weakness in the original formulation of the problem,
i.e., the mismatch in type between the input (a chain) and the output (an spo) of a revision operator,
can be resolved by requiring the ranking on comparisons to be strict.
However, in the present setup this amounts to a less general result, which is why we presented 
our work in this manner.

\section{Concrete Preference Revision Operators}\label{sec:concrete-operator}
Theorems \ref{thm:repr-ties-lr}, 
\ref{thm:repr-ties-rl} and \ref{thm:repr-decisive} articulate an important lesson:
preference revision performed in a principled manner, 
i.e., in accordance with $\ppp{1-4}$ or $\ppp{\mathsf{D}}$ and $\ppp{2-4}$,
involves having preferences over comparisons.
Thus, to obtain concrete operators one must look at ways of 
ranking the comparisons in a preference $\p$.
We sketch here two simple solutions, as proof of concept.

The \emph{trivial assignment $\as^\mathrm{t}$} is defined by taking: 
$$
\drc{i}\approx^\mathrm{t}_\p\drc{j},$$ 
while the \emph{lexicographic assignment $\as^\mathrm{lex}$} is defined by taking: 
$$
\drc{i}<^\mathrm{lex}_\p\drc{j},
$$ 
if $i<j$, for any $\p\in\CHN_V$ and $\drc{i},\drc{j}\in\p$.
Intuitively, the trivial assignment makes all direct comparisons of $\p$ equally desirable, 
while the lexicographic assignment orders them in lexicographic order.

These assignments induce the \emph{trivial} and \emph{lexicographic} operators $\en^\mathrm{t}$ and $\en^\mathrm{lex}$, respectively.
It is straightforward to see that $\le_{\p}^\mathrm{t}$ is a preorder and $<^\mathrm{lex}_\p$ is a linear order, 
prompting the following result.

\begin{proposition}\label{prop:trivial-lex}
	The operator $\en^\mathrm{t}$ satisfies postulates $\ppp{1-4}$.
	The operator $\en^\mathrm{lex}$ satisfies postulates $\ppp{\mathsf{D}}$ and $\ppp{2-4}$.	
\end{proposition}

The following example illustrates that the two operators can give different results on the same input.

\begin{example}\label{ex:trivial-lexi}
	For $\p$ and $\o$ as in Example \ref{ex:preferences-3-items}, 
	the trivial operator ranks all direct comparisons of $\p$, i.e., $(1,2)$ and $(2,3)$, equally, 
	and hence either adds all or none of them to $\o$. 
	Since adding both leads to a cycle, it ends up adding none and hence
	$\p\en^\mathrm{t}\o=(3,1)$. 
	
	The lexicographic assignment ranks $(1,2)$ as better than $(2,3)$, and hence 
	adds $(1,2)$ after which it runs out of options, i.e., $\p\en^\mathrm{lex}\o=(3,1,2)$.
\end{example}

\section{Conclusion}\label{sec:conclusion}
We have presented a model of preference change
according to which revising a preference $\p$ 
goes hand in hand with having preferences over the comparisons of $\p$,
thereby providing a rigorous formal treatment to intuitions found elsewhere 
in the literature
\cite{Sen77,Grune-YanoffH09b}.
Interestingly, the postulates describing preference revision 
are analogous to existing postulates offered for propositional enforcement \cite{HaretWW18}, 
an operation used to model changes in Abstract Argumentation Frameworks (AFs) \cite{Dung1995}.

Our treatment unearthed interesting aspects of preference revision,
such as the issue of coordination between successive instances 
of new preference information (Section \ref{sec:postulates})
and the non-obvious solution to the question of how to rank two comparisons relative
to each other (Section \ref{sec:representation}).
These aspects are taken for granted in regular propositional revision, but prove key
to successful application of revision to the more specialized context of transitive relations
on a set of items, i.e., preference orders.
In this respect, preference revision is akin to revision for fragments of propositional logic
\cite{DelgrandePW18,CreignouHPW18}, and raises the possibility of exporting this approach
to other formalisms in this family. The addition procedure in particular,
lends itself to application in other formalisms by slight tweaking of the acceptance condition,
and could thus supply some interesting lessons for revision in general, 
in particular to revision-like operators for specialized formalisms, such as that of AFs, mentioned above.

There is also ample space for future work with respect to the present framework itself.
To facilitate exposition of the main ideas we imposed certain restrictions on the primary notions.
Lifting these restrictions would yield broader results that would potentially cover more ground and apply
to a more diverse set of inputs. We can consider, for instance, revising strict partial orders in general
(not just linear orders), and using rankings that involve all comparisons of the initial preference order
(not just the direct ones).
As the space of possibilities becomes larger, the choice problems on this space become increasingly more complex
as well. Finding the right conditions under which the choice mechanism corresponds to a set of appealing postulates
requires a delicate balance of many elements, and holds the promise for interesting results.

\clearpage

\bibliographystyle{kr}
\bibliography{bibliography}

\begin{thebibliography}{}

\bibitem[\protect\citeauthoryear{Alchourr{\'{o}}n, G{\"{a}}rdenfors, and
  Makinson}{1985}]{AlchourronGM85}
Alchourr{\'{o}}n, C.~E.; G{\"{a}}rdenfors, P.; and Makinson, D.
\newblock 1985.
\newblock {On the Logic of Theory Change: Partial Meet Contraction and Revision
  Functions}.
\newblock {\em The Journal of Symbolic Logic} 50(2):510--530.

\bibitem[\protect\citeauthoryear{Benthem and Liu}{2014}]{BenthemL14}
Benthem, J., and Liu, F.
\newblock 2014.
\newblock {Deontic Logic and Preference Change}.
\newblock {\em IfCoLog Journal of Logics and their Applications} 1(2):1--46.

\bibitem[\protect\citeauthoryear{Boutilier \bgroup et al\mbox.\egroup
  }{2004}]{BoutilierBDHP04}
Boutilier, C.; Brafman, R.~I.; Domshlak, C.; Hoos, H.~H.; and Poole, D.
\newblock 2004.
\newblock {CP-nets: A Tool for Representing and Reasoning with Conditional
  Ceteris Paribus Preference Statements}.
\newblock {\em Journal of Artificial Intelligence Research (JAIR)} 21:135--191.

\bibitem[\protect\citeauthoryear{Bradley}{2007}]{Bradley07}
Bradley, R.
\newblock 2007.
\newblock The kinematics of belief and desire.
\newblock {\em Synthese} 156(3):513--535.

\bibitem[\protect\citeauthoryear{Cadilhac \bgroup et al\mbox.\egroup
  }{2015}]{CadilhacALB15}
Cadilhac, A.; Asher, N.; Lascarides, A.; and Benamara, F.
\newblock 2015.
\newblock Preference change.
\newblock {\em Journal of Logic, Language and Information} 24(3):267--288.

\bibitem[\protect\citeauthoryear{Chomicki and Song}{2005}]{ChomickiS05}
Chomicki, J., and Song, J.
\newblock 2005.
\newblock {Monotonic and Nonmonotonic Preference Revision}.
\newblock In {\em Proc. {IJCAI} 2005 Multidisciplinary Workshop on Advances in
  Preference Handling}.

\bibitem[\protect\citeauthoryear{Chomicki}{2003}]{Chomicki03}
Chomicki, J.
\newblock 2003.
\newblock Preference formulas in relational queries.
\newblock {\em {ACM} Trans. Database Syst.} 28(4):427--466.

\bibitem[\protect\citeauthoryear{Creignou \bgroup et al\mbox.\egroup
  }{2018}]{CreignouHPW18}
Creignou, N.; Haret, A.; Papini, O.; and Woltran, S.
\newblock 2018.
\newblock {Belief Update in the Horn Fragment}.
\newblock In {\em Proceedings of the Twenty-Seventh International Joint
  Conference on Artificial Intelligence ({IJCAI} 2018)},  1781--1787.

\bibitem[\protect\citeauthoryear{Delgrande, Peppas, and
  Woltran}{2018}]{DelgrandePW18}
Delgrande, J.~P.; Peppas, P.; and Woltran, S.
\newblock 2018.
\newblock {General Belief Revision}.
\newblock {\em Journal of the ACM (JACM)} 65(5):29:1--29:34.

\bibitem[\protect\citeauthoryear{Dell’Acqua and Pereira}{2005}]{DellP2005}
Dell’Acqua, P., and Pereira, L.~M.
\newblock 2005.
\newblock Preference {R}evision {V}ia {D}eclarative {D}ebugging.
\newblock In {\em Portuguese Conference on Artificial Intelligence},  18--28.
\newblock Springer.

\bibitem[\protect\citeauthoryear{Domshlak \bgroup et al\mbox.\egroup
  }{2011}]{DomshlakHKP11}
Domshlak, C.; H{\"{u}}llermeier, E.; Kaci, S.; and Prade, H.
\newblock 2011.
\newblock {Preferences in AI: An Overview}.
\newblock {\em Artificial Intelligence} 175(7-8):1037--1052.

\bibitem[\protect\citeauthoryear{Dung}{1995}]{Dung1995}
Dung, P.~M.
\newblock 1995.
\newblock On the {A}cceptability of {A}rguments and its {F}undamental {R}ole in
  {N}onmonotonic {R}easoning, {L}ogic {P}rogramming and n-{P}erson {G}ames.
\newblock {\em Artif. Intell.} 77(2):321--358.

\bibitem[\protect\citeauthoryear{Ferm{\'{e}} and Hansson}{2018}]{FermeH18}
Ferm{\'{e}}, E.~L., and Hansson, S.~O.
\newblock 2018.
\newblock {\em {Belief Change: Introduction and Overview}}.
\newblock Springer Briefs in Intelligent Systems. Springer.

\bibitem[\protect\citeauthoryear{Frankfurt}{1988}]{Frankfurt88}
Frankfurt, H.~G.
\newblock 1988.
\newblock {Freedom of the Will and the Concept of a Person}.
\newblock In {\em What is a person?} Springer.
\newblock  127--144.

\bibitem[\protect\citeauthoryear{Freund}{2004}]{Freund04}
Freund, M.
\newblock 2004.
\newblock {On the revision of preferences and rational inference processes}.
\newblock {\em Artificial Intelligence} 152(1):105--137.

\bibitem[\protect\citeauthoryear{Gr\"{u}ne-Yanoff and
  Hansson}{2009a}]{Grune-YanoffH09b}
Gr\"{u}ne-Yanoff, T., and Hansson, S.~O., eds.
\newblock 2009a.
\newblock {\em {Preference Change: Approaches from Philosophy, Economics and
  Psychology}}, volume~42 of {\em Theory and Decision Library A}.
\newblock Springer.

\bibitem[\protect\citeauthoryear{Gr{\"u}ne-Yanoff and
  Hansson}{2009b}]{Grune-YanoffH09a}
Gr{\"u}ne-Yanoff, T., and Hansson, S.~O.
\newblock 2009b.
\newblock {From Belief Revision to Preference Change}.
\newblock In {\em Preference Change: Approaches from Philosophy, Economics and
  Psychology}. Springer.
\newblock  159--184.

\bibitem[\protect\citeauthoryear{Gr{\"{u}}ne{-}Yanoff}{2013}]{Grune-Yanoff13}
Gr{\"{u}}ne{-}Yanoff, T.
\newblock 2013.
\newblock Preference change and conservatism: comparing the bayesian and the
  {AGM} models of preference revision.
\newblock {\em Synthese} 190(14):2623--2641.

\bibitem[\protect\citeauthoryear{Hansson}{1995}]{Hansson95}
Hansson, S.~O.
\newblock 1995.
\newblock Changes in preference.
\newblock {\em Theory and Decision} 38(1):1--28.

\bibitem[\protect\citeauthoryear{Hansson}{2017}]{Hansson17}
Hansson, S.~O.
\newblock 2017.
\newblock {Logic of Belief Revision}.
\newblock In Zalta, E.~N., ed., {\em The Stanford Encyclopedia of Philosophy}.
  Metaphysics Research Lab, Stanford University, winter 2017 edition.

\bibitem[\protect\citeauthoryear{Haret, Wallner, and Woltran}{2018}]{HaretWW18}
Haret, A.; Wallner, J.~P.; and Woltran, S.
\newblock 2018.
\newblock {Two Sides of the Same Coin: Belief Revision and Enforcing
  Arguments}.
\newblock In {\em Proceedings of the Twenty-Seventh International Joint
  Conference on Artificial Intelligence ({IJCAI} 2018)},  1854--1860.

\bibitem[\protect\citeauthoryear{Harsanyi}{1955}]{Harsanyi55}
Harsanyi, J.~C.
\newblock 1955.
\newblock {Cardinal Welfare, Individualistic Ethics, and Interpersonal
  Comparisons of Utility}.
\newblock {\em Journal of Political Economy} 63(4):309--321.

\bibitem[\protect\citeauthoryear{Jeffrey}{1974}]{Jeffrey74}
Jeffrey, R.~C.
\newblock 1974.
\newblock Preference among preferences.
\newblock {\em Journal of Philosophy} 71(13):377--391.

\bibitem[\protect\citeauthoryear{Katsuno and Mendelzon}{1992}]{KatsunoM92}
Katsuno, H., and Mendelzon, A.~O.
\newblock 1992.
\newblock {Propositional Knowledge Base Revision and Minimal Change}.
\newblock {\em Artificial Intelligence} 52(3):263--294.

\bibitem[\protect\citeauthoryear{Lang and van~der Torre}{2008}]{LangT08}
Lang, J., and van~der Torre, L. W.~N.
\newblock 2008.
\newblock {Preference Change Triggered by Belief Change: {A} Principled
  Approach}.
\newblock In {\em Proceedings of the 8th International Conference on Logic and
  the Foundations of Game and Decision Theory ({LOFT} 8)},  86--111.

\bibitem[\protect\citeauthoryear{Liu}{2011}]{Liu11}
Liu, F.
\newblock 2011.
\newblock {\em {Reasoning About Preference Dynamics}}, volume 354 of {\em
  Synthese Library}.
\newblock Springer.

\bibitem[\protect\citeauthoryear{Ma, Benferhat, and Liu}{2012}]{MaBL12}
Ma, J.; Benferhat, S.; and Liu, W.
\newblock 2012.
\newblock {Revising Partial Pre-Orders with Partial Pre-Orders: {A} Unit-Based
  Revision Framework}.
\newblock In {\em Proceedings of the 13th International Conference on
  Principles of Knowledge Representation and Reasoning ({KR} 2012)},  633--637.

\bibitem[\protect\citeauthoryear{Nozick}{1994}]{Nozick94}
Nozick, R.
\newblock 1994.
\newblock {\em The Nature of Rationality}.
\newblock Princeton University Press.

\bibitem[\protect\citeauthoryear{Pigozzi, Tsouki{\`{a}}s, and
  Viappiani}{2016}]{PigozziTV16}
Pigozzi, G.; Tsouki{\`{a}}s, A.; and Viappiani, P.
\newblock 2016.
\newblock {Preferences in artificial intelligence}.
\newblock {\em Annals of Mathematics and Artificial Intelligence}
  77(3-4):361--401.

\bibitem[\protect\citeauthoryear{Rossi and Mattei}{2019}]{RossiM19}
Rossi, F., and Mattei, N.
\newblock 2019.
\newblock {Building Ethically Bounded AI}.
\newblock In {\em Proceedings of the 33rd AAAI Conference on Artificial
  Intelligence ({AAAI} 2019)},  9785--9789.

\bibitem[\protect\citeauthoryear{Rossi, Venable, and Walsh}{2011}]{RossiVW11}
Rossi, F.; Venable, K.~B.; and Walsh, T.
\newblock 2011.
\newblock {\em {A Short Introduction to Preferences: Between Artificial
  Intelligence and Social Choice}}.
\newblock Synthesis Lectures on Artificial Intelligence and Machine Learning.
  Morgan {\&} Claypool Publishers.

\bibitem[\protect\citeauthoryear{Russell}{2019}]{Russell2019}
Russell, S.
\newblock 2019.
\newblock {\em Human {C}ompatible: {A}rtificial {I}ntelligence and the
  {P}roblem of {C}ontrol}.
\newblock Penguin.

\bibitem[\protect\citeauthoryear{Sen}{1977}]{Sen77}
Sen, A.~K.
\newblock 1977.
\newblock {Rational Fools: A Critique of the Behavioral Foundations of Economic
  Theory}.
\newblock {\em Philosophy \& Public Affairs}  317--344.

\end{thebibliography}

\end{document}